\documentclass[runningheads]{llncs}
\usepackage[T1]{fontenc}
\usepackage{graphicx}
\usepackage{amsmath}
\usepackage{amssymb}
\usepackage{enumitem}
\usepackage[bookmarks=true]{hyperref}
\hypersetup{
  colorlinks   = true, 
  urlcolor     = blue, 
  linkcolor    = blue, 
  citecolor   = red 
}
\usepackage{graphicx}
\usepackage{listings}
\usepackage{comment}
\usepackage[ruled,linesnumbered]{algorithm2e}
\usepackage[short,c2]{optidef}
\usepackage{adjustbox}
\usepackage[title]{appendix}
\usepackage{wrapfig}
\usepackage{cite}
\usepackage{float}
\usepackage{xcolor}

\definecolor{DarkGreen}{RGB}{1,150,32}
\begin{document}
\title{A Complete Algorithm for a Moving Target Traveling Salesman Problem with Obstacles}
%
%
\author{Anoop Bhat\inst{1} \and
Geordan Gutow\inst{1} \and
Bhaskar Vundurthy\inst{1} \and
Zhongqiang Ren\inst{2} \and
Sivakumar Rathinam\inst{3} \and
Howie Choset\inst{1}}

\authorrunning{A. Bhat et al.}
\titlerunning{MTVG-TSP: A complete algorithm for MT-TSP-O}
%
\institute{Carnegie Mellon University, Pittsburgh PA 15213, USA \and
Shanghai Jiao Tong University, Shanghai, China \and
Texas A\&M University, College Station, TX 77843}
\maketitle              
\begin{abstract}
\emph{The moving target traveling salesman problem with obstacles} (MT-TSP-O) is a generalization of the traveling salesman problem (TSP) where, as its name suggests, the targets are moving. A solution to the MT-TSP-O is a \emph{trajectory} that visits each moving target during a certain time window(s), and this trajectory avoids stationary obstacles. We assume each target moves at a constant velocity during each of its time windows. The agent has a speed limit, and this speed limit is no smaller than any target's speed. This paper presents the first complete algorithm for finding feasible solutions to the MT-TSP-O. Our algorithm builds a tree where the nodes are agent trajectories intercepting a unique sequence of targets within a unique sequence of time windows. We generate each of a parent node's children by extending the parent's trajectory to intercept one additional target, each child corresponding to a different choice of target and time window. This extension consists of planning a trajectory from the parent trajectory's final point in space-time to a moving target. To solve this point-to-moving-target subproblem, we define a novel generalization of a visibility graph called a moving target visibility graph (MTVG). Our overall algorithm is called MTVG-TSP. To validate MTVG-TSP, we test it on 570 instances with up to 30 targets. We implement a baseline method that samples trajectories of targets into points, based on prior work on special cases of the MT-TSP-O. MTVG-TSP finds feasible solutions in all cases where the baseline does, and when the sum of the targets' time window lengths enters a critical range, MTVG-TSP finds a feasible solution with up to 38 times less computation time.

\end{abstract}
\section{Introduction}
Given a set of targets and the travel costs between every pair of targets, the traveling salesman problem (TSP) seeks an order of targets for an agent to visit that minimizes the agent's total travel cost. In the moving target traveling salesman problem (MT-TSP) \cite{helvig2003moving}, the targets are moving through free space, and we seek not only an order of targets, but a trajectory for the agent intercepting each target. The agent's trajectory is subject to a speed limit and must intercept each target within a set of target-specific time intervals, called time windows. We consider the case where travel cost between targets is equal to the travel time: this cost is not fixed \emph{a priori}, as in the TSP, and instead depends on the time at which the agent intercepts each target. Prior work on the MT-TSP assumes that the agent's speed limit is no smaller than the speed of any target \cite{helvig2003moving,philip2023CStar,de2019experimental,wang2023moving}, and we make the same assumption in our work. When there are moving targets as well as obstacles for the agent to avoid, we refer to the problem as the moving target traveling salesman problem with obstacles (MT-TSP-O), shown in Fig. \ref{fig:intro_figure}.

\begin{figure}
\vspace{-0.5cm}
\centering
\includegraphics[width=1\textwidth]{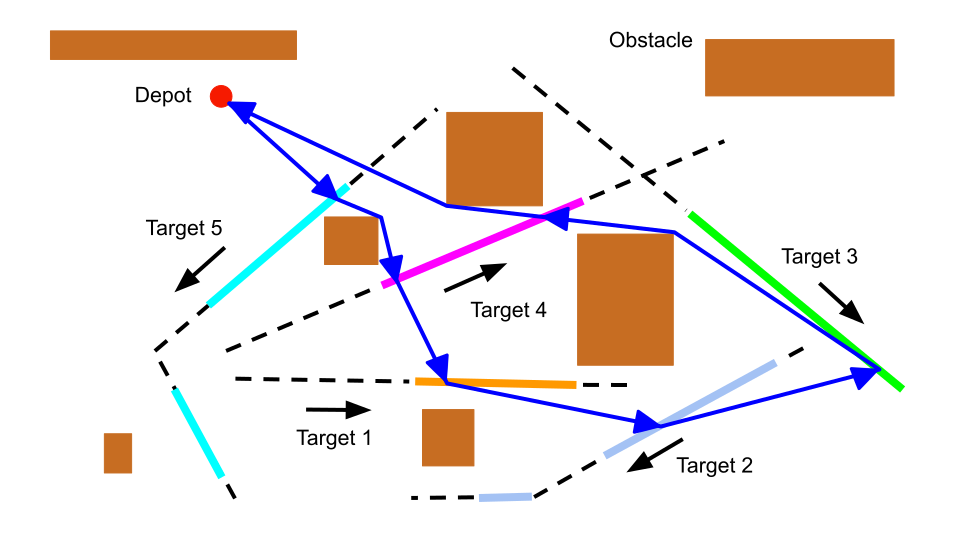}
\vspace{-0.5cm}
\caption{Targets move along trajectories with piecewise-constant velocities, which can be intercepted by agent during time windows depicted in bold colored lines. Agent’s trajectory shown in dark blue avoids obstacles, intercepts each target within its time window, and returns to start location (depot).} \label{fig:intro_figure}
\vspace{-0.5cm}
\end{figure}

Two properties we desire for an MT-TSP-O algorithm are completeness\footnote{Completeness refers to an algorithm's guarantee on finding a feasible solution when a problem instance is feasible or reporting infeasible in finite time otherwise \cite{choset2005principles, lavalle2006planning}.} and optimality. No algorithm for the MT-TSP-O in the literature has either of these properties. Guaranteeing completeness is complicated by the fact that even the problem of finding a feasible solution is NP-complete, since the MT-TSP-O generalizes the TSP with time windows (TSP-TW) \cite{savelsbergh1985local}. In this paper, we present the first complete algorithm for the MT-TSP-O.

Simpler cases of the MT-TSP-O have been addressed in the literature, with completeness guarantees in some cases. For example, in the absence of obstacles, \cite{philip2024mixedinteger} provides a complete and optimal solver for the MT-TSP assuming targets move at constant velocities. \cite{stieber2022multiple} provides a complete and optimal method when targets have piecewise-constant velocities. Heuristics are presented in \cite{bourjolly2006orbit, choubey2013moving, groba2015solving, ucar2019meta, jiang2005tracking, de2019experimental, 6760409, marlow2007travelling, wang2023moving} for variants of the MT-TSP, only guaranteeing completeness in the absence of time windows. In the presence of obstacles, there is one related work \cite{Li2019RendezvousPlanning} that considers the case where the agent is restricted to travel along a straight-line path when moving from one target to the next, providing an incomplete algorithm. A generic approach to solving these special cases of the MT-TSP-O is to sample the trajectories of targets into points, find agent trajectories between every pair of points, then select a sequence of points to visit by solving a generalized traveling salesman problem (GTSP) \cite{philip2023CStar, stieber2022DealingWithTime, Li2019RendezvousPlanning}. This approach is not complete, since it may only be possible for the agent to intercept some target at a time that is in between two of the sampled points in time representing the target's trajectory.

We develop a new algorithm for the MT-TSP-O that guarantees completeness. Our algorithm, called MTVG-TSP, leverages a novel generalization of a visibility graph, which we call a moving target visibility graph (MTVG), which enables us to plan a trajectory from a starting point in space-time to a moving target. Given a sequence of time windows of targets, we can find a minimum-time agent trajectory intercepting each target within its specified time window via a sequence of A* searches, each on a MTVG. In particular, we can do so without sampling any target's trajectory, avoiding the limitations of prior work. By performing a higher level search for a sequence of time windows and computing an agent trajectory for each generated sequence, MTVG-TSP finds an MT-TSP-O solution if one exists. We extensively test our algorithm on problem instances with up to 30 targets, varying the length and number of time windows, and we compare our algorithm's computation time to a method based on prior work \cite{philip2023CStar,stieber2022DealingWithTime,Li2019RendezvousPlanning} that samples trajectories of targets into points. We demonstrate that when the sum of the time window lengths for each target enters a critical range, the sampled-points method requires an excessive number of sample points to find a feasible solution, while our method finds solutions relatively quickly.

\section{Problem Setup}\label{sec:problem_setup}
We consider a single agent and $N_\tau$ targets, all moving on a 2D plane ($\mathbb{R}^2$). The trajectory of target $i\in [N_\tau]$\footnote{For a positive integer $x$, $[x]$ denotes the set $\{1,2,\cdots,x\}$.} is denoted as $\tau_i: \mathbb{R}^+ \rightarrow \mathbb{R}^2$. Each target has a set of $N_i$ time windows $\{w_{i, 1}, w_{i, 2}, \dots, w_{i, N_i}\}$, where $w_{i, j} = [t_{i, j}^0, t_{i, j}^f]$ is the $j^{th}$ time window of target $i$. Target $i$ moves at a constant velocity within each of its time windows, but its velocity may be different for each time window. Given a final time $T^f$, denote the trajectory for the agent as $\tau_A : [0, T^f] \rightarrow \mathbb{R}^2$. $\tau_A$ must start and end at a given point called the depot denoted as $d$ with position $p_{d} \in \mathbb{R}^2$. The agent can move in any direction with a speed at most $v_{max}$.

Let $\{O_1, O_2, \dots O_{N_O}\}$ denote the set of obstacles where $N_O$ denotes the number of obstacles. We define $\Psi(t^0, t^f)$ as the set of all the feasible agent trajectories defined on the time interval $[t^0, t^f]$ such that for any time $t\in [t^0, t^f]$, the agent satisfies the speed constraint and its position never enters the interior of any obstacle. We assume that within target $i$'s time windows, target $i$ does not move with speed greater than $v_{max}$ and does not enter the interior of any obstacle.

We say that a trajectory $\tau_A$  for the agent \textit{intercepts} target $i\in [N_\tau]$ if there exists a time $t$ such that for some $j \in [N_{i}]$, $t \in w_{i, j}$ and $\tau_A(t) = \tau_i(t)$. We define the MT-TSP-O as the problem of finding a final time $T^f$ and an agent trajectory $\tau_A \in \Psi(0, T^f)$ such that $\tau_A$ starts and ends at the depot, intercepts each target $i\in [N_\tau]$, and $T^f$ is minimized. 

\section{MTVG-TSP Algorithm}\label{sec:algorithm}
The MTVG-TSP algorithm interleaves a higher-level search on a {\it time window graph} and a lower-level search on a \textit{moving target visibility graph} (MTVG). The nodes in the time window graph, called {\it window-nodes}, each represent either the depot or a pairing of a target with one of its time windows. A feasible solution to the MT-TSP-O corresponds to a cycle in the time window graph containing the depot and exactly one window-node per target. Our algorithm finds such a cycle by building a \textit{trajectory tree}, where each \textit{tree-node} contains a sequence of window-nodes and an associated agent trajectory. A tree-node's children are each generated in two steps. First, we append a window-node to the end of the tree-node's window-node sequence. Then we extend the tree-node's agent trajectory to intercept the appended window-node's target. We perform this trajectory extension via the MTVG. In particular, we construct the MTVG by augmenting a standard visibility graph with the window-node whose target we aim to intercept. Extending a tree-node's trajectory consists of planning a path in the MTVG from the final point in the tree-node's trajectory to the added window-node. When we have a tree-node with a trajectory that intercepts all targets and returns to the depot, we have a solution to the MT-TSP-O.

\subsection{Window-Nodes}\label{sec:window_node_def}
Since we will be constructing two novel graphs, the time window graph and the moving target visibility graph, each containing a common type of node called a \textit{window-node}, we define window-nodes here. A window-node $s = (i, t^0_{i, j}, t^f_{i, j})$ is an association of a target $i$ with one of its time windows $w_{i, j} = [t^0_{i, j}, t^f_{i, j}]$. The set of all window nodes is $V_{tw} = \{(0, 0, \infty)\} \cup \bigcup\limits_{i \in [N_\tau]}\bigcup\limits_{j \in [N_{i}]}\{(i, t_{i, j}^0, t_{i, j}^f)\}$. In addition to having a window-node for every possible pair of target and time window, we have a window-node $s_d = (0, 0, \infty)$ for the depot. The depot is viewed as a fictitious target 0 with $\tau_0(t) = p_d$ for all $t$. For each window-node $s = (i, t_{i, j}^0, t_{i, j}^f)$, we define the functions $\text{targ}(s) = i$, $t^0(s) = t_{i, j}^0$, and $t^f(s) = t_{i, j}^f$. We say an agent trajectory $\tau_A$ \textit{intercepts} $s \in V_{tw}$ at time $t \in [t^0(s), t^f(s)]$ if $\tau_A(t) = \tau_{\text{targ}(s)}(t)$.

\subsection{Initial Visibility Computations}\label{sec:init_visibility}
Next, we describe two initial data structures needed to construct the moving target visibility graph. The first is a standard visibility graph $G_{vis} = (V_{vis}, E_{vis})$ \cite{lozano1979algorithm}. Let $V_O \subseteq \mathbb{R}^2$ be the set of convex vertices of all the obstacles\footnote{A convex obstacle vertex is is a vertex where the internal angle between the two incident edges is less than $\pi$ radians. Using only convex vertices in a visibility graph reduces graph size without discarding shortest paths through the graph \cite{liu1992path}}. The set of nodes $V_{vis}$ in $G_{vis}$ contains vertices in $V_O$, the depot position, and the positions of each target's trajectory at the beginning and end of each of its time windows, i.e. $V_{vis} := V_O \cup \{p_d\} \cup \{\tau_{\text{targ}(s)}(t^0(s)) : s \in V_{tw}\setminus \{s_d\}\} \cup \{\tau_{\text{targ}(s)}(t^f(s)) : s \in V_{tw}\setminus \{s_d\}\}$. We draw an edge from $q \in V_{vis}$ to $q' \in V_{vis}$ if $q'$ is contained in the visibility polygon of $q$, denoted as $\text{vpoly}(q)$ (computed using CGAL \cite{cgal:eb-24a}).

The second data structure encodes visibility relationships between points in space and window-nodes. In particular, for each $q \in V_{vis}$, $s \in V_{tw}$, we compute a \textit{visible interval set} $\text{vis}(q, s)$, containing every interval $I \subseteq [t^0(s), t^f(s)]$ such that for all $t \in I$, we have $\tau_{\text{targ}(s)}(t) \in \text{vpoly}(q)$. We illustrate $\text{vis}(q, s)$ in Fig. \ref{fig:vis_def} (a). The set of all visible interval sets is $\Lambda_{vis} = \{\text{vis}(q, s) : \forall q \in V_{vis}, s \in V_{tw}\}$.

\begin{figure}
\vspace{-0.5cm}
\centering
\includegraphics[width=\textwidth]{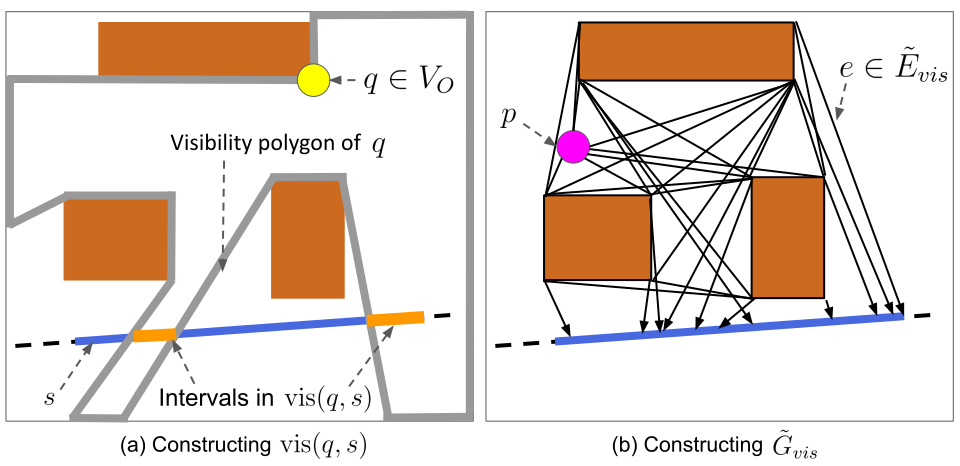}
\vspace{-0.5cm}
\caption{(a) We compute $\text{vis}(q, s)$ for all $q \in V_O$ and $s \in V_{tw}$. (b) In the moving target visibility graph (MTVG) associated with $p$ and $s$, we draw an edge from $q$ to $s$ if $\text{vis}(q, s) \neq \emptyset$. All edges between position-nodes are bidirectional, but edges from position-nodes to $s$ are unidirectional: there are no edges leaving $s$. The positions of endpoints of edges on $s$ are drawn arbitrarily, since the agent's position after traversing edge $(q, s)$ depends on the time when the agent leaves $q$.} \label{fig:vis_def}
\vspace{-0.5cm}
\end{figure}

\subsection{Moving Target Visibility Graph}\label{subsec:mtvg}
A recurring subproblem in our algorithm is to find a trajectory from a point $(p, T)$ in space-time that intercepts a particular window-node $s$ in minimum time. For each of these subproblems, we define a graph $\tilde{G}_{vis} = (\tilde{V}_{vis}, \tilde{E}_{vis})$ called a \textit{moving target visibility graph} (MTVG). The set of nodes is $\tilde{V}_{vis} = V_{vis} \cup \{p, s\}$, where nodes in $V_{vis} \cup \{p\}$ are called \textit{position-nodes}, and $s$ is the window-node we aim to intercept. The set of edges is $\tilde{E}_{vis} = E_{vis} \cup E_p \cup E_s$. We construct $E_p$ by drawing edges from $p$ to all $q \in V_{vis}$ such that $q \in \text{vpoly}(p)$; if we already have $p \in V_{vis}$, we can skip this step, since all possible edges from $p$ to $q \in V_{vis}$ already exist in $E_{vis}$. Constructing $E_s$ consists of two steps. First, we compute $\text{vis}(p, s)$; if $p \in V_{vis}$, we can skip this step, because we computed $\text{vis}(p, s)$ in Section \ref{sec:init_visibility}. Second, we draw an edge to $s$ from any position-node $q$, including $p$, such that $\text{vis}(q, s) \neq \emptyset$, as shown in Fig. \ref{fig:vis_def} (b).

The cost of edge $(u, v) \in \tilde{E}_{vis}$ depends on $u$, $v$, and a time variable $t$ representing the time at which the agent departs from node $u$. We denote the cost of $(u, v)$ at time $t$ as $\tilde{c}_{vis}(u, v, t)$. For edges between position-nodes $u$ and $v$, the time variable is not used: the edge cost is simply equal to the agent's minimum travel time from $u$ to $v$: $\tilde{c}_{vis}(u, v, t) = \frac{\|u - v\|_2}{v_{max}} \; \forall t$. For any edge $e_s = (q, s) \in E_s$, the time variable is used to compute the cost as follows:
\begin{align}
    \tilde{c}_{vis}(q, s, t) &= \min\limits_{I \in \text{vis}(q, s)} SFT(q, t, s, I)\label{eqn:cost_of_last_edge}
\end{align}
where $SFT$ stands for shortest feasible travel, as in \cite{philip2023CStar}. "Shortest" refers to shortest time. The SFT from point $(q, t)$ to window-node $s$ on interval $I$ is the optimal cost of optimization Problem \ref{optprob:sft_problem}:
\begin{mini!}
{t_s \in I}{t_s - t}{\label{optprob:sft_problem}}{SFT(q, t, s, I) = }
\addConstraint{\frac{\|\tau_{\text{targ}(s)}(t_s) - q\|}{t_s - t} \leq v_{max}}{}
\end{mini!}
Problem \ref{optprob:sft_problem} computes the minimum travel time of a straight-line trajectory starting from $(q, t)$ and intercepting $s$ within interval $I$, and can be solved in closed-form using methods from \cite{philip2023CStar}. The SFT computation does not consider obstacle-avoidance, but in our case, it does not need to: any straight-line trajectory from $q$ to $s$ intercepting $s$ within some $I \in \text{vis}(q, s)$ is already obstacle-free.

After constructing $\tilde{G}_{vis}$, we find a minimum-time trajectory from $(p, T)$ to $s$ via an A* search \cite{hart1968AFormalBasis} from $p$ to $s$ on $\tilde{G}_{vis}$, shown in Alg. \ref{alg:point_to_moving_target_planning_alg}. Since edge costs in $\tilde{G}_{vis}$ encode travel time, the g-value $g(v)$ for a node $v$ is the shortest travel time out of all paths to $v$ that A* has explored so far. $T + g(v)$ is then the earliest known arrival time to node $v$. In Line \ref{algline:gcand_computation}, we use this arrival time to compute edge costs from $v$ to its successors. The heuristic $h(q)$ for a position-node $q$ is the Euclidean distance from $q$ to the spatial line segment defined by $s$, divided by $v_{max}$, underestimating the travel time from $q$ to $s$. We set $h(s) = 0$. When A* finds a path to $s$, the ConstructTrajectory function (Line \ref{algline:reconstruct_trj}) performs a standard backpointer traversal to obtain a path $Q = (p, q^1, \dots, q^{N - 1}, s)$ through $\tilde{G}_{vis}$, moves the agent at max speed between each position-node in $Q$, and finally executes the straight-line trajectory associated with $\tilde{c}_{vis}(q^{N - 1}, s, T + g(q^{N - 1}))$. The result is a minimum-time trajectory from $(p, T)$ to $s$, if one exists. If such a trajectory does not exist, the condition $g_{cand}(v') \leq t^f(s) - T$ on Line \ref{algline:check_before_add_to_open} for adding $v'$ to OPEN ensures the search terminates, described in Section \ref{sec:theory}.

\vspace{0.1cm}
\begin{algorithm}[H]\label{alg:point_to_moving_target_planning_alg}
\SetKwFunction{PointToMovingTargetSearch}{PointToMovingTargetSearch}
\SetKwFunction{ConstructMTVG}{ConstructMTVG}
\caption{A* search from initial point $(p, T)$ to window-node $s$}
\SetKwProg{Fn}{Function}{:}{}
\Fn{\PointToMovingTargetSearch{$p, T, s, G_{vis}, \Lambda_{vis}$}} {
    $\tilde{G}_{vis} = (\tilde{V}_{vis}, \tilde{E}_{vis}) = $ \ConstructMTVG($p, s, G_{vis}, \Lambda_{vis}$)\;
    OPEN = []\;
    CLOSED = []\;
    Insert $p$ into OPEN with $f(p) = 0$\;
    Set $g(v) = \infty$ for all $v \in \tilde{V}_{vis}$ with $v \neq p$.
    Set $g(p) = 0$.
    
    \While{\upshape OPEN is not empty and $f(s) > \min\limits_{v \in \text{OPEN}}f(v)$} {
      Remove $v$ with smallest $f(v)$ from OPEN\;\label{algline:before_expand}
      Insert $v$ into CLOSED\;
    
      \For{\upshape $v'$ in $v$.successors()} {
        $g_{cand}(v') = g(v) + \tilde{c}_{vis}(v, v', T + g(v))$\;\label{algline:gcand_computation}
        \If{\upshape $g_{cand}(v') < g(v')$ and $v' \notin$ CLOSED and $g_{cand}(v') \leq t^f(s) - T$} {\label{algline:check_before_add_to_open}
          $g(v') = g_{cand}(v')$\;
          Insert $v'$ into OPEN with $f(v') = g(v') + h(v')$\;
        }
      }
    }
    
    \If{$f(s) \neq \infty$} {
        \Return ConstructTrajectory($s$), $g(s) + T$, FEASIBLE\;\label{algline:reconstruct_trj}
    }
    \Return NULL, $\infty$, INFEASIBLE\;\label{eqn:return_infeasible}
} 
\end{algorithm}

\subsection{Time Window Graph}\label{sec:twg}
\sloppy The MTVG defined in Section \ref{subsec:mtvg} enables finding a minimum-time trajectory intercepting a single window-node. We will need to chain these trajectory computations together to intercept a sequence of window-nodes, containing one window-node per target. To determine this sequence, we define a {\it time window graph} denoted as $G_{tw} = (V_{tw}, E_{tw})$. The set of nodes in $G_{tw}$ is the set of all window nodes $V_{tw}$ from Section \ref{sec:window_node_def}. \sloppy We add an edge $(u, v)$ to $E_{tw}$ if there exists an agent trajectory that intercepts $u$ at time $t_u$ and $v$ at time $t_v$ with $t_u \leq t_v$, satisfying speed limit and obstacle avoidance constraints. In particular, we search for such a trajectory that departs at the latest feasible departure time from $u$ to $v$, denoted as $LFDT(u, v)$. Extending the idea from \cite{philip2023CStar}, $LFDT(u, v)$ is the optimal cost of optimization Problem \ref{optprob:lfdt_problem}: 
\begin{maxi!}
{\substack{t_u \in [t^0(u), t^f(u)],\\\tau_A}}{t_u}{\label{optprob:lfdt_problem}}{LFDT(u, v) = }
\addConstraint{\tau_A \in \Psi(t_u, t^f(v))\label{eqn:lfdt_avoid_obs_and_speed_limit}}{}
\addConstraint{\tau_A(t_u) = \tau_{\text{targ}(u)}(t_u) \label{eqn:lfdt_start_on_u}}{}
\addConstraint{\tau_A(t^f(v)) = \tau_{\text{targ}(v)}(t^f(v))\label{eqn:lfdt_end_on_v}}{}.
\end{maxi!}

\begin{remark}
Constraint \eqref{eqn:lfdt_end_on_v} requires the agent to meet $v$ at its end time $t^f(v)$, raising the question of whether the agent could depart later from $u$ if we relaxed \eqref{eqn:lfdt_end_on_v} and only required the agent to meet $v$ at some time $t_v \in [t^0(v), t^f(v)]$. Since we assumed $\tau_{\text{targ}(v)}$ neither exceeds the agent's maximum speed nor enters the interior of an obstacle during $\text{targ}(v)$'s time window, this relaxation would not let the agent depart later from $u$. If the agent can depart at time $t_u$ and meet $\text{targ}(v)$ at some time $t_v \in [t^0(v), t^f(v)]$, the agent can follow $\tau_{\text{targ}(v)}$ from time $t_v$ to time $t^f(v)$, thereby meeting $v$ at time $t^f(v)$ as well.
\end{remark}

Problem \ref{optprob:lfdt_problem} seeks a trajectory that starts by intercepting a window-node and terminates at a prescribed point. We transform Problem \ref{optprob:lfdt_problem} via a time reversal to instead seek a trajectory that starts at a prescribed point and terminates by intercepting a window-node. We do so because the transformed problem can be solved using Alg. \ref{alg:point_to_moving_target_planning_alg}. The transformation defines a fictitious node $\underline{u} = (-\text{targ}(u), -t^f(u), -t^0(u))$ associated with fictitious target $\text{targ}(\underline{u}) = -\text{targ}(u)$. The trajectory of $\text{targ}(\underline{u})$ is $\tau_{-\text{targ}(u)}(t) = \tau_{\text{targ}(u)}(-t)$. Consider a reversed time variable $\underline{t} = -t$. An agent trajectory that departs as late as possible from $u$, measured in conventional time $t$, arrives at $\underline{u}$ as early as possible, measured in reversed time $\underline{t}$. We find $LFDT(u, v)$ by solving the transformed Problem \ref{optprob:lfdt_problem_transform}:
\begin{mini!}
{\substack{\underline{t}_u \in [t^f(\underline{u}), t^0(\underline{u})],\\\underline{\tau}_A}}{\underline{t}_u}{\label{optprob:lfdt_problem_transform}}{LFDT(u, v) = -}
\addConstraint{\underline{\tau}_A \in \Psi(-t^f(v), \underline{t}_u)\label{eqn:lfdt_avoid_obs_and_speed_limit_transform}}{}
\addConstraint{\underline{\tau}_A(t_u) = \tau_{\text{targ}(\underline{u})}(\underline{t}_u) \label{eqn:lfdt_start_on_u_transform}}{}
\addConstraint{\underline{\tau}_A(-t^f(v)) = \tau_{\text{targ}(v)}(t^f(v)).\label{eqn:lfdt_end_on_v_transform}}{}
\end{mini!}
We solve Problem \ref{optprob:lfdt_problem_transform} using Alg. \ref{alg:point_to_moving_target_planning_alg}, and if we find a feasible solution, we draw an edge from $u$ to $v$ in $G_{tw}$ and store $LFDT(u, v)$ with that edge.

\subsection{Trajectory Tree}\label{subsec:feas_trj_gen}
Any feasible agent trajectory $\tau_A$ for the MT-TSP-O must intercept each target during one of its time windows before returning to the depot. This means that $\tau_A$ intercepts a sequence of window-nodes $(s^1, s^2, \dots, s^{N_\tau})$ containing exactly one window-node per target, implying the existence of the cycle $S = (s_d, s^1, s^2, \dots, s^{N_\tau}, s_d)$ in $G_{tw}$. In this paper, we use the depth-first search (DFS) in Alg. \ref{alg:dfs} to find a feasible trajectory and its corresponding cycle (if one exists). While we apply DFS in this work to quickly find feasible solutions, other search methods can be used as well, e.g. best-first search to find the global optimum.

Alg. \ref{alg:dfs} maintains a stack that stores tuples $(S, \tau_A, T)$, where $S$ is a path (sequence of window-nodes) through $G_{tw}$, and $(\tau_A, T)$ is a partial solution to the MT-TSP-O, in that $(\tau_A, T)$ starts at the depot, avoids obstacles, satisfies the speed limit, and intercepts each window-node in $S$ in order. As the search proceeds, we construct a tree of these tuples, which we call a \textit{trajectory tree}. We refer to tuples $(S, \tau_A, T)$ as \textit{tree-nodes}. Each loop iteration begins by popping a tree-node $(S, \tau_A, T)$ from the stack (Line \ref{algline:pop}), then iterating over the successor window-nodes of $(S, T)$ (Line \ref{algline:get_succ_tw}), where $s' \in G_{tw}$ is a successor window-node of $(S, T)$ if and only if all of the following conditions hold:
\begin{enumerate}
    \item $(s, s') \in E_{tw}$ and $T \leq LFDT(s, s')$, $s$ being the terminal window-node of $S$.\label{enumitem:succ_tw_cond1}
    \item $\text{targ}(s') \neq \text{targ}(s)$ for any $s$ in $S$.
    \item If $s' = s_d$, then $S$ contains one window-node per target.
\end{enumerate}

\vspace{0.5cm}
\begin{algorithm}[H]\label{alg:dfs}
\caption{Constructing a cycle through $G_{tw}$ and a trajectory intercepting each node in the cycle. See Alg. \ref{alg:point_to_moving_target_planning_alg} for the PointToMovingTargetSearch function. See Section \ref{subsec:feas_trj_gen} for details on the Lookahead function.}

\SetKwFunction{DFS}{DFS}
\SetKwFunction{AugmentVisibilityGraph}{AugmentVisibilityGraph}
\SetKwFunction{PointToMovingTargetSearch}{PointToMovingTargetSearch}
\SetKwFunction{Lookahead}{Lookahead}
\SetKwFunction{GetSuccessorWindowNodes}{GetSuccessorWindowNodes}
\SetKwFunction{Append}{Append}
\SetKwFunction{ConcatenateTrajectories}{ConcatenateTrajectories}
\SetKwProg{Fn}{Function}{:}{}
\SetKwComment{Comment}{// }{}

\Fn{\DFS{$S, G_{vis}, \Lambda_{vis}$}} {
    STACK = [ $(s_d, NULL, 0)$ ]\; \label{algline:init_stack}
    \While{\upshape STACK is not empty}{\label{algline:check_stack_empty}
        $(S, \tau_A, T) = $ STACK.pop()\;\label{algline:pop}
        SUCCESSORS = []\;
        \For{\upshape $s'$ in $\GetSuccessorWindowNodes(S, T)$}{\label{algline:get_succ_tw}
            $\bar{\tau}_{A}'$, $T'$, status = \PointToMovingTargetSearch($\tau_A(T), T, s', G_{vis}, \Lambda_{vis}$)\;\label{algline:pt_moving_target_result}
            $S' = $ \Append(S, s')\;\label{algline:append_sprime}
            \If{\upshape \Lookahead($S', T'$) is INFEASIBLE}{\label{algline:fail_check}
              continue\;
            }
            $\tau_A' = $ \ConcatenateTrajectories($\tau_A, \bar{\tau}_A')$\;\label{algline:concatenate_trj}
            \If{\upshape $s'$ is $s_d$}{
              return $\tau_A', T'$ FEASIBLE\;\label{algline:return_mt_tsp_o_soln}
            }
            SUCCESSORS.append($(S', \tau_A', T')$)\;\label{algline:add_succ}
        }
        \Comment{Add successors to stack in order of decreasing final time}
        SUCCESSORS.sortLargestToSmallestT()\;\label{algline:sort_by_T}
        \For{\upshape $(S', \tau_A', T')$ in SUCCESSORS}{
            STACK.push(($S', \tau_A', T'$))\;
        }\label{algline:end_stack_push}
    }
    return NULL, $\infty$, INFEASIBLE\;\label{algline:term_infeas}
} 
\end{algorithm}
\vspace{0.5cm}

For each successor window-node $s'$, we generate a successor tree-node $(S', \tau_A', T')$ by generating a trajectory $\bar{\tau}_A'$ that begins at $(\tau_A(T), T)$ and intercepts $s'$. We plan $\bar{\tau}_A'$ using Alg. \ref{alg:point_to_moving_target_planning_alg}, obtaining final time $T' = T + g(s')$. Condition \ref{enumitem:succ_tw_cond1} in the definition of a successor window-node ensures that we always find a trajectory in this step. Next, we generate a path $S'$ through $G_{tw}$ (or a cycle through $G_{tw}$, if $s' = s_d$) by appending $s'$ to $S$ (Alg. \ref{alg:dfs}, Line \ref{algline:append_sprime}). Then we perform a check denoted as Lookahead in Line \ref{algline:fail_check}. In particular, let $\Gamma_{unvisited}$ be the set of targets unvisited by $S' = (s_d, s^1, s^2, \dots, s')$:
\begin{align}
    \Gamma_{unvisited} = [N_\tau] \setminus \{\text{targ}(s^i) : i \in [|S'| - 1]\}.
\end{align}
where $|S'|$ is the length of $S'$. If inequality \eqref{eqn:lookahead_condition} holds for some $i \in \Gamma_{unvisited}$,
\begin{align}
    T' > \max\limits_{s'' \in \{s''' \in V_{tw} : \text{targ}(s''') = i\}}LFDT(s', s'')\label{eqn:lookahead_condition}.
\end{align}
then the agent has arrived at $s'$ too late to intercept target $i$ in the future, making it impossible to intercept the set of unvisited targets. Thus we do not add $S'$ as a successor to $S$. This Lookahead check is analogous to Test 1 in \cite{Dumas1995OptimalAlgorithm}. While it is not needed for completeness, it reduces computation time.

If the Lookahead check succeeds, we concatenate $\bar{\tau}_A'$ with $\tau_A$ to obtain trajectory $\tau_A'$, and thereby a partial solution $(\tau_A', T')$ corresponding to $S'$ (Line \ref{algline:concatenate_trj}. If $s' = s_d$, then $(\tau_A', T')$ is a feasible solution to the MT-TSP-O and we return (Line \ref{algline:return_mt_tsp_o_soln}). Otherwise, we add $(S', \tau_A', T')$ to the list of successors of $S$ (Line \ref{algline:add_succ}). After obtaining all successors, we push the successors onto the stack in order of decreasing final time (Lines \ref{algline:sort_by_T}-\ref{algline:end_stack_push}), so the next popped tree-node will be the successor of $(S, \tau_A, T)$ with the earliest final time.

\section{Theoretical Analysis}\label{sec:theory}
In this section, we state MTVG-TSP's completeness theorems and sketch the proofs, providing full proofs in Appendix \ref{sec:appendix} in the supplementary material. 

\begin{theorem}\label{theorem:optimally_solve_subproblem}
Alg. \ref{alg:point_to_moving_target_planning_alg} finds a minimum-time trajectory beginning at $(p, T)$ and intercepting $s$, if one exists.
\end{theorem}
Theorem \ref{theorem:optimally_solve_subproblem}'s proof is similar to the admissibility proof of A* \cite{hart1968AFormalBasis}, differing since the cost of an edge $(q, s)$ in the MTVG depends on $g(q)$ in the A* search. In general, such path-dependent edge costs can make it necessary for an optimal path in a graph to take a suboptimal path to some intermediate node. The bulk of Theorem \ref{theorem:optimally_solve_subproblem}'s proof lies in showing that this is never necessary in the MTVG.

\begin{definition}
Let $(\tau_A^*, T^*)$ be a solution to the MT-TSP-O. $(\tau_A, T)$ is a \textnormal{prefix} of $(\tau_A^*, T^*)$ if $\tau_A(t) = \tau_A^*(t)$ for all $t \in [0, T]$ and $T \leq T^*$.
\end{definition}

\begin{theorem}\label{theorem:completeness_thm_feas}
If an MT-TSP-O instance is feasible, MTVG-TSP finds a feasible solution.
\end{theorem}
Proving Theorem \ref{theorem:completeness_thm_feas} requires showing that Alg. \ref{alg:dfs} terminates, and that it terminates by returning a feasible MT-TSP-O solution, as opposed to terminating on Line \ref{algline:term_infeas} due to an empty stack. Termination is guaranteed because Alg. \ref{alg:dfs}'s search tree is finite. We show that termination is not caused by an empty stack via an induction argument proving that there is always a prefix of a feasible MT-TSP-O solution on the stack. We detail the induction step here.

Suppose a tuple $(S, \tau_A, T)$ is popped from the stack with $(\tau_A, T)$ a prefix of some feasible MT-TSP-O solution $(\tau_A^*, T^*)$. Let $p = \tau_A(T)$. Of the window-nodes intercepted by $\tau_A^*$ but not $\tau_A$, let $s'$ be the window-node intercepted earliest. Since $\tau_A^*$ travels feasibly from $(p, T)$ to $s'$, arriving at some time ${T^*}'$, Theorem \ref{theorem:optimally_solve_subproblem} implies that Alg. \ref{alg:point_to_moving_target_planning_alg} generates a trajectory $\bar{\tau}_A'$ from $(p, T)$ to $s'$, arriving with $T' \leq {T^*}'$. Under the assumptions that the targets move no faster than the agent's maximum speed and do not enter the interior of obstacles during their time windows, it is feasible for an agent trajectory to follow $\tau_{\text{targ}(s')}$ from $T'$ to ${T^*}'$. Therefore we can construct a feasible MT-TSP-O solution $(\tau_A^{**}, T^*)$ by having the agent follow $\tau_A$ until time $T$, follow $\bar{\tau}_A'$ until $T'$, follow $\tau_{\text{targ}(s')}$ from $T'$ to ${T^*}'$, and follow $\tau_A^*$ from ${T^*}'$ to $T^*$. We push $(S', \tau_A', T')$ onto the stack, where $\tau_A'$ is the concatenation of $\tau_A$ with $\bar{\tau}_A'$, and $(\tau_A', T')$ is a prefix of $(\tau_A^{**}, T^*)$.

\begin{theorem}\label{theorem:completeness_thm_infeas}
If an MT-TSP-O instance is infeasible, MTVG-TSP terminates and reports infeasible in finite time.
\end{theorem}

Termination of Alg. \ref{alg:dfs} follows from its search tree being finite, so the remaining claim to prove is that Alg. \ref{alg:point_to_moving_target_planning_alg} terminates when performing LFDT computations before executing Alg. \ref{alg:dfs}. To do so, we use the condition on Line \ref{algline:check_before_add_to_open} $g_{cand}(v') \leq t^f(s) - T$ to prevent the search from generating nodes that are reached after the end time of $s$, which bounds the number of steps through $\tilde{G}_{vis}$ away from $p$ the search will ever explore. We then apply this bound within the termination proof of A* \cite{hart1968AFormalBasis}.

\begin{wrapfigure}[13]{r}{0.5\textwidth}
\centering
\vspace{-2.25cm}
\includegraphics[width=0.5\textwidth]{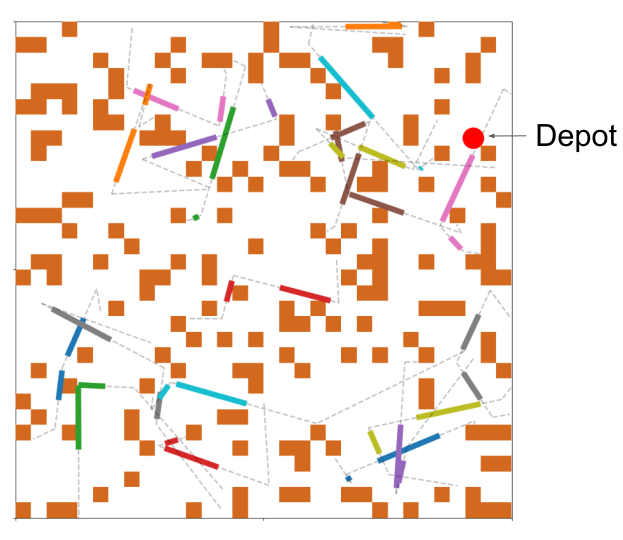}
\vspace{-0.5cm}
\caption{Example 20-target instance. Each target's time window lengths sum to 26 s.}
\label{fig:example_instance}
\vspace{-0.5cm}
\end{wrapfigure}

\section{Experiments}
We ran all experiments on an Intel i9-9820X 3.3GHz CPU with 128 GB RAM. As a baseline, we implemented a sampled-points based method detailed in Section \ref{subsec:baseline}, based on prior work that solve special cases of the MT-TSP-O \cite{philip2023CStar, stieber2022DealingWithTime, Li2019RendezvousPlanning}. We test MTVG-TSP and the baseline on 570 problem instances, where an instance consists of a depot location, trajectories and time windows of targets, and an obstacle grid. An example instance is shown in Fig. \ref{fig:example_instance}, and we show a solution to an example instance in the video in the supplementary material. For each method, we measured the computation time for the method to obtain its first feasible solution for each instance, setting an upper limit of 300 s.

\subsection{Baseline}\label{subsec:baseline}
For our baseline, based on \cite{Li2019RendezvousPlanning, stieber2022DealingWithTime, philip2023CStar}, we sample the trajectory of each target within its time windows into points in space-time, such that if we concatenated all of a target's time windows, the points would be spaced uniformly in time. Next, we plan the shortest obstacle-free path in space between each pair of points using \cite{cui2017compromise}, then convert the path into a trajectory where the agent moves at max speed along the path, then waits with zero velocity at the final position until the final time. If the agent cannot reach the final position by the final time, travel between those two points is infeasible. After computing these trajectories, we pose a generalized traveling salesman problem (GTSP) to find a sequence of points to visit. We then formulate the GTSP as an integer linear program (ILP) as in \cite{laporte1987generalized}, but without subtour elimination constraints. Subtours are only possible if trajectories of two targets intersect exactly in space and time, and we ensure this does not occur. We solve the ILP using Gurobi, obtaining a sequence of points, then concatenate the trajectories between every consecutive pair of points in the sequence to obtain a MT-TSP-O solution.

Since we are not guaranteed to get a feasible solution for a fixed number of sample points, we initialize the algorithm with 10 points per target. If the ILP is infeasible, we increase the number of points by 10 and attempt to solve the ILP again. We repeat this process until the ILP is feasible, then take the first feasible solution Gurobi produces. The computation time reported for an instance is the sum of computation times for all attempted numbers of sample points.

\subsection{Generating Problem Instances}
We generated two sets of instances, corresponding to Experiment 1 (Section \ref{subsec:vary_tw_len}) and Experiment 2 (Section \ref{subsec:vary_num_tw}). In Experiment 1, we varied the number of targets from 10 to 30 in increments of 10 and varied the sum of each target's time window lengths from 2 s to 50 s in increments of 4, keeping the number of time windows per target fixed to 2. In Experiment 2, we varied the number of targets as in Experiment 1 and the number of time windows per target from 1 to 6, keeping the sum of lengths fixed. We generated 10 instances for each choice of experiment parameters. When varying the sum of a target's time window lengths, we randomly generated the instances with the longest windows first, then randomly shortened the time windows to generate more instances. When varying the number of windows per target, we first generated instances with 1 window with length equal to 22 s, then randomly split this single window into multiple windows. We generated the instances prior to shortening and splitting time windows as follows to ensure all instances were feasible. First, we randomly sampled an occupancy grid with 20\% of cells occupied. Then we initialized the agent at a random depot location $p_d$ in free space. Finally we repeated the following steps $N_\tau$ times, starting with $i = 1$, $p^0 = p_{d}$, and $t^0 = 0$:
\begin{enumerate}
    \item Sample an interception position $p^i$ in free space for target $i$.
    \item Plan a path in space from $p^{i - 1}$ to $p_i$ using \cite{cui2017compromise}.
    \item Move the agent at a speed $\beta v_{max}$ along the path until the end of the path at $p^i$, obtaining an arrival time $t^i = t^{i - 1} + d^i/(\beta v_{max})$, where $d^i$ is the distance traveled along the path. Here, we use $\beta = 0.99$ to make the instances challenging without making them borderline infeasible.
    \item Sample a piecewise linear trajectory for target $i$ such that the number of linear segments equals the specified number of time windows, the trajectory duration is greater than the specified sum of window lengths, and the trajectory arrives at position $p^i$ at time $t^i$. For each segment of the trajectory, we sample the velocity direction uniformly at random, then select the speed uniformly at random from the range $[\frac{v_{max}}{8}, \frac{v_{max}}{4}]$. This range is from \cite{philip2024mixedinteger}, which studies the MT-TSP. After creating this trajectory, we randomly sample a subset of each segment's time interval to create the time windows.
\end{enumerate}

\subsection{Experiment 1: Varying Sum of Time Window Lengths}\label{subsec:vary_tw_len}
In this experiment, we varied the number of targets and the sum of each target's time window lengths. The results in Fig. \ref{fig:vary_tw_len} show that as we increase the number of targets, we see wider and wider ranges for the sum of window lengths where MTVG-TSP outperforms the sampled-points method in median and maximum computation time. In these critical ranges, the sampled-points method's computation time is large because it needs a large number of points to find a feasible solution, as shown in Fig. \ref{fig:points_needed_vary_tw}, leading to a large underlying integer program. In Appendix \ref{sec:usable_fraction} in the supplementary material, we show that these peaks occur when there are large intervals in some target's time windows where no feasible MT-TSP-O solution intercepts the target. In these cases, it is difficult to sample a point in one of the \textit{usable intervals} where some feasible solution does intercept the target, since the combined length of these usable intervals is small relative to the combined length of the time windows.

MTVG-TSP's median computation time varies less significantly than the sampled-points method's, though MTVG-TSP's maximum computation time peaks in the same regions as the sampled-points method's. For example, in Fig. \ref{fig:vary_tw_len} (b), consider the 20-target instance with the largest peak in MTVG-TSP's runtime, occuring when the sum of time window lengths equals 14 s. In this instance, we found that if we decreased the sum of lengths to 10 s, MTVG-TSP found a solution intercepting the same sequence of time windows as it did for the 14 s instance, but with 65\% less computation time. The reason for this phenomenon is that in the 14 s instance, there are several paths through $G_{tw}$ that do not exist in the 10 s instance, adding branches to Alg. \ref{alg:dfs}'s search tree. Profiling showed that in the 14 s instance, Alg. \ref{alg:dfs} spent 60\% of its time exploring these additional branches. The fact that it returned the same sequence of time windows as in the 10 s instance indicates that the additional paths through $G_{tw}$ are useless: they cannot be part of a cycle corresponding to a feasible MT-TSP-O solution, only adding to computation time. On the other hand, we found that increasing the sum of time window lengths from 14 s to 18 s caused MTVG-TSP to find a different time window sequence than in the 14 s instance, and twice as quickly. In this case, increasing time window lengths added useful paths to $G_{tw}$.

\begin{figure*}
    \centering
    \vspace{-0.1cm}
    \includegraphics[width=1\textwidth]{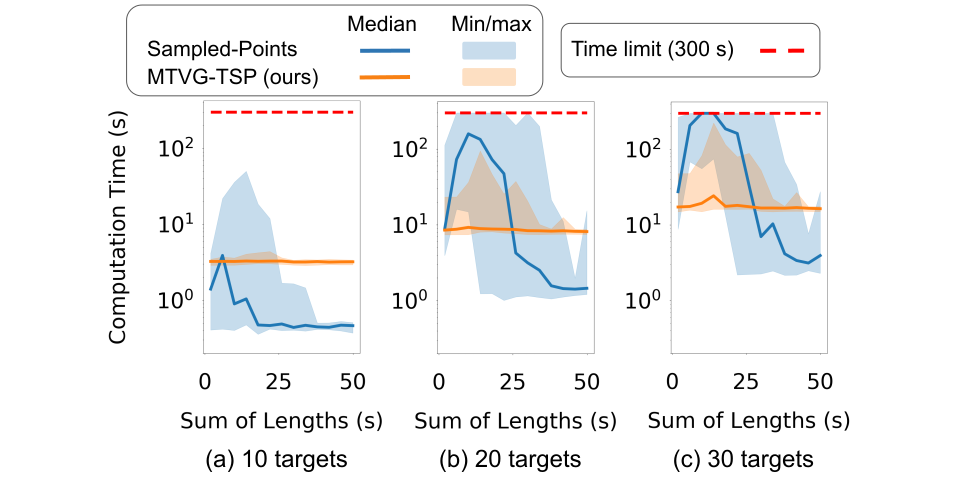}
    \vspace{-0.6cm}
    \caption{Time for each algorithm to compute a feasible solution, varying the sum of time window lengths per target while fixing the number of time windows to two. The sampled-points method reached the time limit in 11\% of instances without finding a feasible solution. In these cases,  the computation time is set equal to the time limit. MTVG-TSP did not reach the time limit in any instance.}
    \label{fig:vary_tw_len}
    \vspace{-0.1cm}
\end{figure*}

\begin{figure}
    \centering
    \includegraphics[width=1\textwidth]{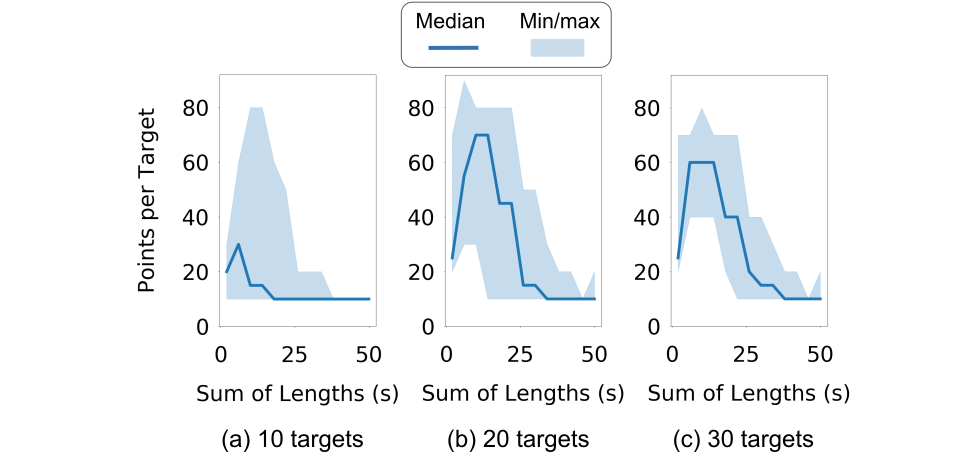}
    \vspace{-0.6cm}
    \caption{Maximum attempted number of points per target used by sampled-points method. In instances where the method found a feasible solution, we report the number of points used to obtain the solution. In instances where the method timed out, we report the number of points upon timeout.}
    \label{fig:points_needed_vary_tw}
    \vspace{-0.2cm}
\end{figure}

\renewcommand{\arraystretch}{1.8}
\begin{table}
\vspace{-0.2cm}
\caption{Comparison of the solution costs $T_{MTVG}$ from our method against the costs $T_{PT}$ from the sampled-points method. Appendix \ref{sec:additional_plots} provides plots of these costs.}\label{table:cost_comparison}
\centering
\begin{tabular}{|@{\quad}l@{\quad}|@{\quad}l@{\quad}|@{\quad}l@{\quad}|@{\quad}l@{\quad}|}
\hline
 & Median & Min & Max \\
\hline
$\frac{T_{MTVG} - T_{PT}}{T_{PT}}*100\%$ & \textcolor{DarkGreen}{\textbf{-0.46\%}} & \textcolor{DarkGreen}{\textbf{-28\%}} & \textcolor{orange}{\textbf{34\%}}\\
\hline
\end{tabular}
\vspace{-1cm}
\end{table}

In Table \ref{table:cost_comparison}, we compare the costs between the methods in instances where the sampled-points method found a solution\footnote{MTVG-TSP found a solution in all instances.}. The median percent difference in cost between the methods is small, indicating that both methods often provide similar solution quality. The range of percent differences is large, since we take the first feasible solution from each method as opposed to solving to optimality. 

\subsection{Experiment 2: Varying Number of Time Windows per Target}\label{subsec:vary_num_tw}
In this experiment, we varied the number of targets and number of time windows. Since the sampled-points method does not depend on the number of time windows, only the total set of times covered by the time windows (which we are not varying), we run the sampled-points method for the instances with one window, then show the same runtimes for instances with multiple time windows. As shown in Fig. \ref{fig:vary_num_tw}, when we decrease the number of time windows, MTVG-TSP outperforms the sampled-points method in median and max computation time. Both methods' computation times increase with the number of targets. To explain the trends in MTVG-TSP's computation time, we divide its computation time between its three major components: initial visibility computations (Section \ref{sec:init_visibility}), time window graph construction (Section \ref{sec:twg}), and trajectory tree construction (Section \ref{subsec:feas_trj_gen}). Fig. \ref{fig:mtvg_tsp_time_breakdown} shows that computation time for all components increases with the number of targets and number of time windows per target. Increasing either of these quantities increases the total number of time windows, leading to more window-nodes in the time window graph. This requires adding more nodes into the initial visibility graph corresponding to endpoints of targets' trajectories within their time windows, and computing more visible interval sets. Both operations increase initial visibility computation time. More window-nodes also leads to more pairwise LFDT computations when constructing the time window graph. Finally, a larger time window graph leads to a wider and deeper trajectory tree, leading to a larger number of MTVG constructions and searches during trajectory tree construction.
\begin{figure*}
    \centering
    \vspace{-0.3cm}
    \includegraphics[width=1\textwidth]{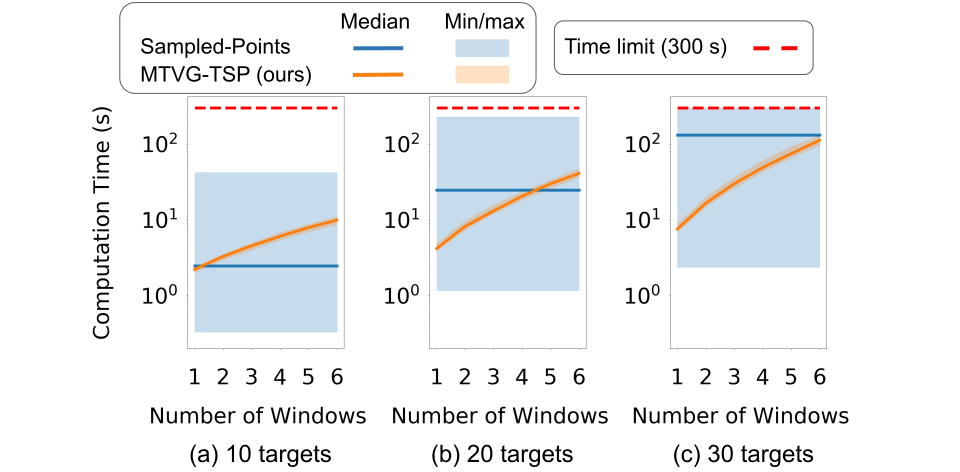}
    \vspace{-0.6cm}
    \caption{Time for each algorithm to compute a feasible solution, varying the number of time windows.}
    \label{fig:vary_num_tw}
    \vspace{-0.1cm}
\end{figure*}
\begin{figure*}
    \centering
    \includegraphics[width=\textwidth]{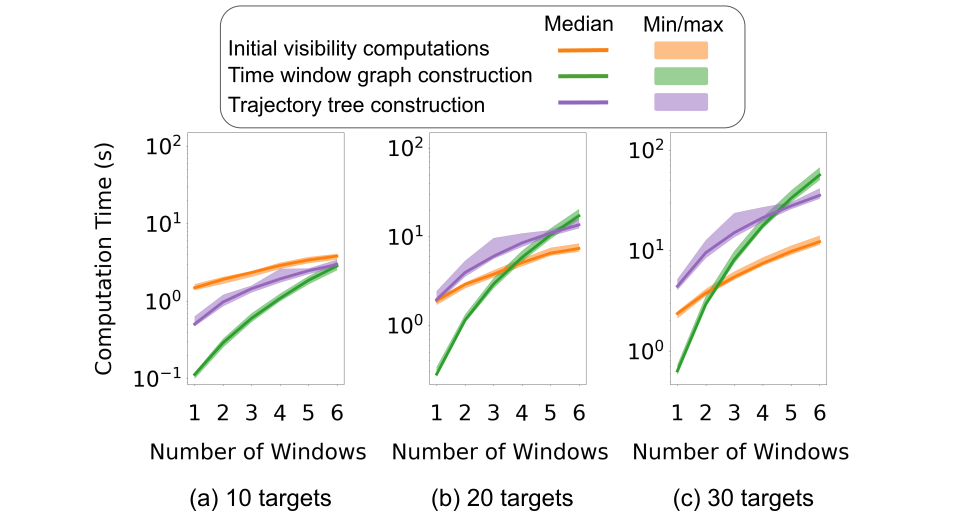}
    \vspace{-0.6cm}
    \caption{Breaking down computation time for MTVG-TSP.}
    \label{fig:mtvg_tsp_time_breakdown}
\end{figure*}

We similarly break down the timing for the sampled-points method in this experiment in Appendix \ref{sec:additional_plots}, Fig. \ref{fig:point_discr_timing_breakdown}, dividing time between GTSP graph construction (finding trajectories between all pairs of points) and GTSP solve time. GTSP solve time exceeds graph construction time when we have 20 or more targets.

\section{Conclusion}
In this paper, we presented MTVG-TSP, a complete algorithm for the moving target traveling salesman problem with obstacles, leveraging a novel graph called a moving target visibility graph (MTVG). We showed that for a range of time window lengths, our algorithm takes less median and maximum time to find feasible solutions than prior methods. Future directions for this work are to incorporate kinodynamic constraints on the agent and involve multiple agents.

\begin{credits}
\subsubsection{\ackname} This material is partially based upon work supported by the National Science Foundation under Grant No. 2120219 and 2120529. Any opinions, findings, and conclusions or recommendations expressed in this material are those of the author(s) and do not necessarily reflect the views of the National Science Foundation.

\subsubsection{\discintname}
The authors have no competing interests to declare that are
relevant to the content of this article.
\end{credits}

\clearpage
%
%
%
\bibliographystyle{splncs04}
%
\bibliography{refs}

\clearpage

\begin{subappendices}
\renewcommand{\thesection}{\Alph{section}}%

\section{Characterizing Instance Difficulty Using the Minimum Usable Fraction}\label{sec:usable_fraction}
We define the \textit{usable time set} of target $i$ as the set $\bar{w}_{i} \subseteq \bigcup\limits_{j \in [N_i]} w_{i,j}$ such that for all $t \in \bar{w}_i$, some feasible solution to the MT-TSP-O intercepts target $i$ at time $t$. We can partition the usable time set uniquely into a set of disjoint subintervals $\bar{W}_i = \{\bar{w}_{i, 1}, \bar{w}_{i, 2}, \dots, \bar{w}_{i, |\bar{W}_i|}\}$ of target $i$'s time windows, where each $\bar{w}_{i, j}$ is called a \textit{maximal usable interval}. We define the \textit{usable fraction} for target $i$ as the sum of lengths of target $i$'s maximal usable intervals, divided by the sum of lengths of target $i$'s time windows. For a given instance, we call the \textit{minimum usable fraction} the minimum over all targets of their usable fractions. Fig. \ref{fig:usable_fraction} shows that the minimum usable fraction is a measure of difficulty of an instance. As the minimum usable fraction decreases, a method that represents trajectories of targets using sample points, as described in Section \ref{subsec:baseline}, tends to need a larger number of points to find a feasible solution.

To compute the usable fraction, we must enumerate each target's maximal usable intervals. Each maximal usable interval $\bar{w}_{i, j}$ for target $i$ can be represented uniquely as a union of several possibly overlapping \textit{sequence-specific usable intervals} $\bar{w}_{i, j, k}$, where each $k$ corresponds to a unique set of feasible MT-TSP-O solutions that intercepts a particular sequence of window-nodes. If we enumerate all sequences of window-nodes corresponding to a feasible MT-TSP-O solution, find the sequence-specific usable interval for each target, for each sequence, then merge any of a target's overlapping sequence-specific usable intervals, we have the maximal usable intervals for each target. Then we can compute the usable fraction.

Now we show how to compute the sequence-specific usable intervals for each target, given a sequence of window-nodes $S = (s^1, s^2, \dots, s^{N_\tau})$. First, we apply Alg. \ref{alg:point_to_moving_target_planning_alg} recursively to determine for each $s^i$ the earliest possible time $t^{E, i}$ that an agent trajectory could intercept $s^i$ while respecting the order of $S$. In particular, we start by running Alg. \ref{alg:point_to_moving_target_planning_alg} from point $(p_D, 0)$ to $s_1$, obtaining arrival time $t^{E, 1}$. Then for each $i$ from 1 to $N_{\tau} - 1$, we run Alg. \ref{alg:point_to_moving_target_planning_alg} from point $(\tau_{\text{targ}(s^i)}(t^{E, i}), t^{E, i})$ to $s^{i + 1}$ and obtain arrival time $t^{E, i}$. We can similarly run Alg. \ref{alg:point_to_moving_target_planning_alg} backwards in time to determine the latest possible time $t^{L, i}$ a trajectory could intercept each $s^i$ while respecting the order of $S$. We start by setting $t^{L, N_\tau} = t^f(s^{N_\tau})$. Then for each $i$ from $N_{\tau}$ to $2$, we construct a fictitious window-node $s_F = (-i, t^{L, N_\tau}, t^{L, i})$, with $\tau_{-i}(t^{L, i}) = \tau_{\text{targ}(s^i)}(t^{L, i})$ and compute $t^{L, i - 1} = LFDT(s^{i - 1}, s_F)$. $[t^{E, i}, t^{L, i}]$ is the sequence-specific usable interval for $\text{targ}(s^i)$ corresponding to the set of feasible MT-TSP-O solutions that intercepts the window-nodes in $S$ in order.

In Fig. \ref{fig:usable_fraction}, we perform the above procedure to compute all sequence-specific usable intervals, then the maximal usable intervals, then the usable fraction for each target, and finally the the minimum usable fraction for each instance, in the instances from Experiment 1 (Section \ref{subsec:vary_tw_len}) with 10 targets. In the instances with 20 or more targets, we found that enumerating all window-node sequences became intractable. We plot the minimum usable fraction against the number of points used by the sampled-points method in that instance. The number of points is large when the minimum usable fraction is small, since it means that for some target, most of its sample points will land in some \textit{unusable} interval within the target's time windows, where no feasible solution intercepts the target.

\begin{figure*}
    \centering
    \includegraphics[width=0.5\textwidth]{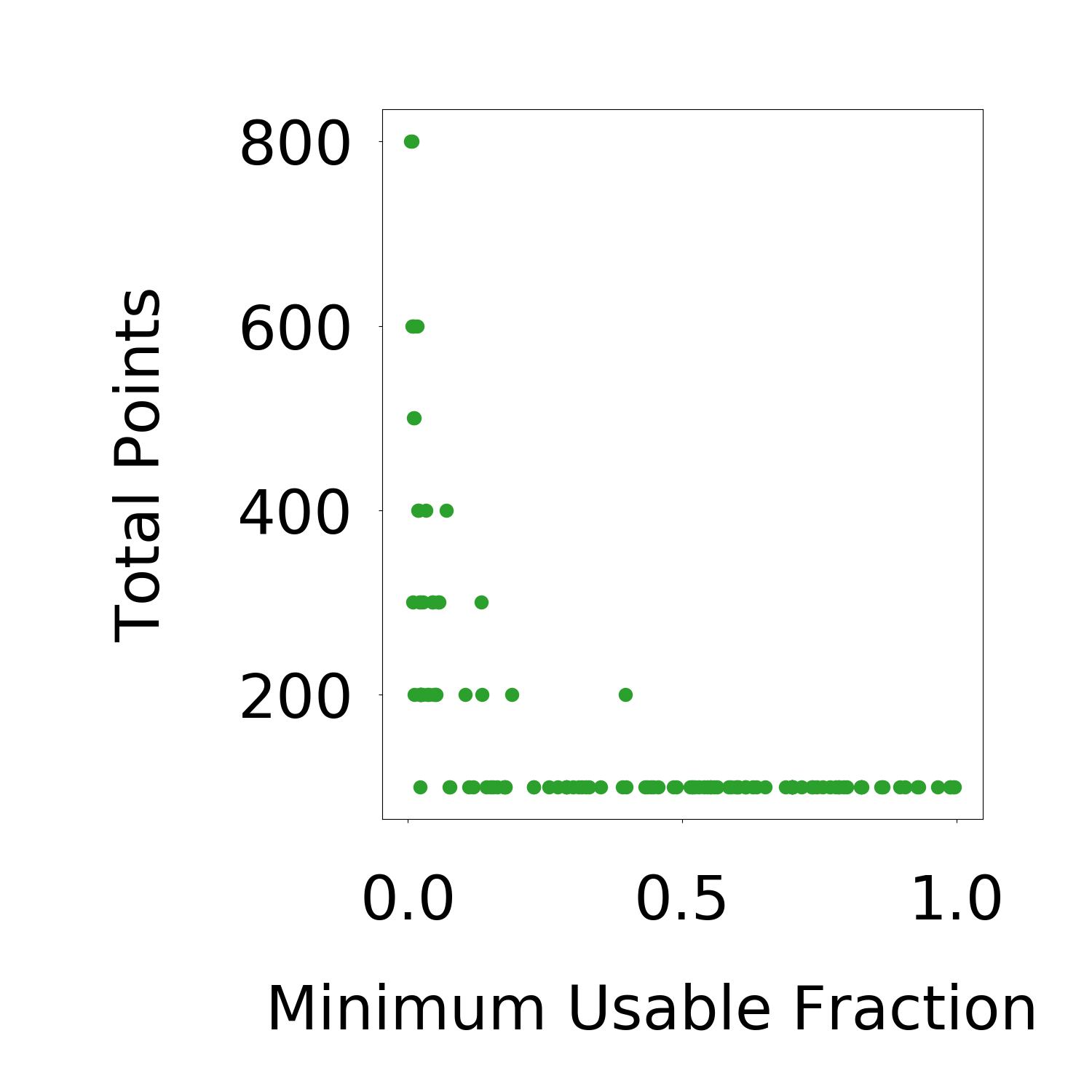}
    \caption{Minimum usable fraction vs. number of points used by sampled-points method, which is the number of targets times the number of points per target. Each point represents one of the instances with 10 targets from Experiment 1 (Section \ref{subsec:vary_tw_len}). When the minimum usable fraction is small, the sampled-points method needs many points.}
    \label{fig:usable_fraction}
\end{figure*}

\clearpage

\section{Additional Plots}\label{sec:additional_plots}
\begin{figure*}
    \centering
    \vspace{-1cm}
    \includegraphics[width=\textwidth]{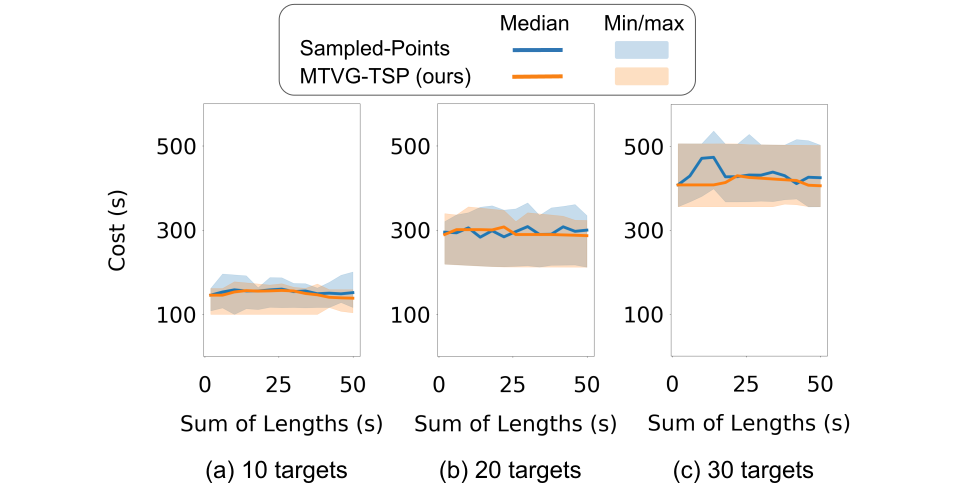}
    \vspace{-0.65cm}
    \caption{Comparing costs of feasible solutions from MTVG-TSP and the sampled-points method, in cases where the sampled-points method found a feasible solution.}
    \label{fig:comparing_costs}
\end{figure*}

\begin{figure*}
    \centering
    \vspace{-1cm}
    \includegraphics[width=\textwidth]{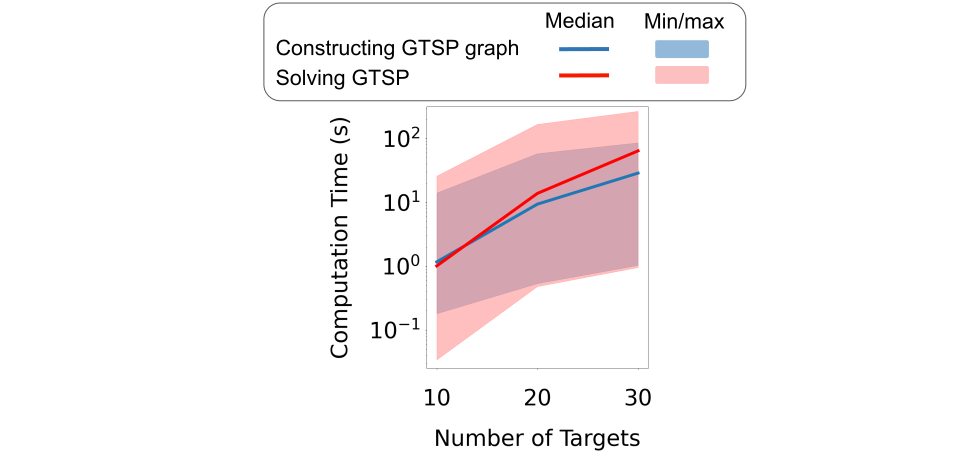}
    \vspace{-0.65cm}
    \caption{Breaking down computation time for sampled-points method. As we increase the number of targets, the median time to solve the GTSP for a sequence of points overtakes the time to construct the GTSP graph, where graph construction consists of finding trajectories between all pairs of points.}
    \label{fig:point_discr_timing_breakdown}
\end{figure*}

\clearpage

\section{Proof of Completeness}\label{sec:appendix}
\subsection{Proof of Theorem \ref{theorem:optimally_solve_subproblem}}
In this section, we provide intermediate results and finally a proof for Theorem \ref{theorem:optimally_solve_subproblem} from the main paper, appearing as Theorem \ref{theorem:feasible_and_terminate_implies_optimal} here, stating that Alg. \ref{alg:point_to_moving_target_planning_alg} finds a minimum-time trajectory from a point $(p, T) \in \mathbb{R}^2 \times \mathbb{R}$ to a goal window-node $s$. Formally, Alg. \ref{alg:point_to_moving_target_planning_alg} solves the following optimization problem:
\begin{mini!}
{\tau_A', T'}{T'}{}{\hypertarget{pt_MT}{\text{(PT-MT)}}\qquad T^* =}
\addConstraint{\tau_A(T) = p\label{eqn:point_to_mt_init_cond}}{}
\addConstraint{\tau_A \in \Psi(T, T')\label{eqn:point_to_mt_avoid_obs_and_speed_limit}}{}
\addConstraint{\tau_A(T') = \tau_{\text{targ}(s)}(T') \label{eqn:point_to_mt_intercept}}{}
\end{mini!}

Any path $Q$ through $\tilde{G}_{vis}$ from $p$ to $s$ can be converted to a feasible solution for \hyperlink{pt_mt}{(PT-MT)} using the ConstructTrajectory function in Alg. \ref{alg:point_to_moving_target_planning_alg}, described in Section \ref{subsec:feas_trj_gen}. Next, we show that we can solve \hyperlink{pt_mt}{\textnormal{(PT-MT)}} optimally by searching for an optimal path $Q$ through $\tilde{G}_{vis}$, where we define optimality with respect to a path cost function $c_{path}$. For a path $Q = (q^0, q^1, \dots, q^{N - 1})$ containing only position-nodes,
\begin{align}
    c_{path}(Q) &= \sum\limits_{i \in [N - 1]}\frac{\|q^i - q^{i - 1}\|_2}{v_{max}}.
\end{align}
For a path $Q = (q^0, q^1, \dots, q^{N - 1}, s)$ ending with window-node $s$, we have
\begin{align}
    c_{path}(Q) &= c_{path}(Q[:N]) + SFT(q^{N - 1}, T + c_{path}(Q[:N]), s)\label{eqn:path_cost_to_mt}
\end{align}
where we abused notation and defined the following:
\begin{align}
    SFT(q, t, s) = \min\limits_{I \in \text{vis}(q, s)} SFT(q, t, s, I).
\end{align}
Here, the notation $Q[:N]$ indicates the length $N$ \textit{prefix} of $Q$, with $Q[:N] = (q^0, q^1, \dots, q^{N-  1})$.

\begin{lemma}\label{lemma:feas_soln_implies_finite_cost_path_in_G_tw}
Suppose problem \hyperlink{pt_mt}{\textnormal{(PT-MT)}} is feasible. A path $Q = (p, q^1, \dots q^{N - 1}, s)$ through $\tilde{G}_{vis}$ from $p$ to $s$ exists with cost $T^* - T$, where $T^*$ is the optimal cost of \hyperlink{pt_mt}{\textnormal{(PT-MT)}}. Such a path $Q$ is an optimal path through $\tilde{G}_{vis}$.
\end{lemma}

\begin{proof}
Let $(\tau_A^*, T^*)$ be one of the optimal solutions to \hyperlink{pt_mt}{(PT-MT)}, and in particular, the optimal solution with minimal distance traveled\footnote{Since \hyperlink{pt_mt}{(PT-MT)}, minimizes time, not distance, not all optimal solutions will have the same distance traveled.}. Let $p^* = \tau_{\text{targ}(s)}(T^*)$. $\tau_A^*$ necessarily follows a shortest collision-free path $\pi$ in space from $p$ to $p^*$. Path $\pi$ is necessarily polygonal, with all inner vertices being obstacle vertices \cite{deBerg2000computational}. Therefore all inner vertices of $\pi$ are in $V_O \subseteq V_{vis} \subseteq \tilde{V}_{vis}$. The source vertex of $\pi$ is $p$, and $p \in \tilde{V}_{vis}$. The destination vertex of $\pi$ is $p^*$, $p^*$ is a position on the trajectory segment corresponding to $s$, and $s \in \tilde{V}_{vis}$. The travel time of $\tau_A^*$ from $p$ to $q^{N - 1}$ is $\sum\limits_{i \in [N - 1]}\frac{\|q^{i} - q^{i - 1}\|_2}{v_{max}}$, with $q^0 = p$. Let $Q = (p, q^1, \dots, q^{N - 1}, s)$ be the path in $\tilde{G}_{vis}$ beginning with $p$, containing the inner vertices of $\pi$ in sequence, and finishing at $s$. $c_{path}({Q[:N]})$ is the same as the travel time of $\tau_A^*$ from $p$ to $q^{N - 1}$, so $T + c_{path}({Q[:N]})$ is the same as the arrival time of $\tau_A^*$ at $q^{N - 1}$.

Suppose for the sake of contradiction that the travel time of $\tau_A^*$ from $q^{N - 1}$ to $p^*$ is $T_{con}$ with $T_{con} > SFT(q^{N - 1}, T + c_{path}(Q[:N]), s)$, making its arrival time at $s$ equal to $T + c_{path}({Q[:N]}) + T_{con}$. Let $\bar{\tau}_A$ be the trajectory obtained by converting $Q$ into a feasible \hyperlink{pt_mt}{\textnormal{(PT-MT)}} solution using the ConstructTrajectory function. $\bar{\tau}_A$ arrives at $s$ at time $T + c_{path}({Q[:N]}) + SFT(q^{N - 1}, T + c_{path}(Q[:N]), s)$, which is less than the arrival time of $\tau_A^*$. This is a contradiction because the optimality of $\tau_A^*$ precludes any trajectory from arriving at $s$ earlier than $\tau_A^*$ does. Thus the travel time of $\tau_A^*$ from $q^{N - 1}$ to $p^*$ is $SFT(q^{N - 1}, T + c_{path}(Q[:N]), s)$, making its travel time from $p$ to $s$ equal $c_{path}(Q)$. Therefore $c_{path}(Q) = T^* - T$. We have now proven the first part of the lemma, since we have shown that some path through $\tilde{G}_{vis}$ exists from $p$ to $s$ with cost $T^* - T$.

Suppose for the sake of contradiction that some path $Q'$ from $p$ to $s$ through $\tilde{G}_{vis}$ exists with $c_{path}(Q') < T^* - T$. $Q'$ can be converted to a feasible solution for \hyperlink{pt_mt}{(PT-MT)} with cost $T + c_{path}(Q')$, meaning there is a feasible solution to \hyperlink{pt_mt}{(PT-MT)} with lower cost than $T + T^* - T = T^*$. This is a contradiction, because $T^*$ is the optimal cost of \hyperlink{pt_mt}{(PT-MT)}. Thus a path through $\tilde{G}_{vis}$ cannot have cost less than $T^* - T$, making any path with cost $T^* - T$ optimal. \qed
\end{proof}

We need the following Lemma in the proof for Lemma \ref{lemma:path_exists_with_optimal_prefix}.
\begin{lemma}\label{lemma:start_earlier_implies_arrive_at_same_time_worst_case}
Let $p \in \mathbb{R}^2$, and let $s$ be a window-node. Let $t^1, t^2 \in \mathbb{R}$ be times when $s$ is visible to $p$ such that $t^1 \leq t^2$. The following holds:
\begin{align}
    t^1 + SFT(p, t^1, s) \leq t^2 + SFT(p, t^2, s).\label{eqn:SFT_Q_leq_SFT_PN-1}
\end{align}
\end{lemma}

\begin{proof}
Let
\begin{align}
    t_*^1 &= t^1 + SFT(p, t^1, s)\label{eqn:aux_var1}\\
    t_*^2 &= t^2 + SFT(p, t^2, s)\label{eqn:aux_var2}.
\end{align}
The solution to an SFT problem from $(p, t^2)$ to $s$ corresponds to an agent trajectory $\tau_{A, 2}^*$ that moves along a straight line in space from $(p, t^2)$ to $(\tau_{\text{targ}(s)}(t_*^2), t_*^2)$. Define the agent trajectory $\tau_{A, 1}^*$ as follows:
\begin{align}
    \tau_{A, 1}^*(t) &= \begin{cases}
        p, & t \in [t^1, t^2]\\
        \tau_{A, 2}^*, & t \in [t^2, t^2_*]
    \end{cases}
\end{align}
$\tau_{A, 1}^*$ starts at $(p, t^1)$ and meets $s$ at time $t_*^2$, achieving travel time
\begin{align}
    t^2_* - t^1
\end{align}
meaning that the shortest feasible travel time from $(p, t^1)$ to $s$ satisfies
\begin{align}
    SFT(p, t^1, s) \leq t^2_* - t^1\label{eqn:sft_statement}.
\end{align}
Adding $t^1$ to both sides of \eqref{eqn:sft_statement} and applying \eqref{eqn:aux_var2}, we have
\begin{align}
    t^1 + SFT(p, t^1, s) \leq t^2 + SFT(p, t^2, s).\qed
\end{align}
\end{proof}

\begin{lemma}\label{lemma:path_exists_with_optimal_prefix}
Suppose there is a path through $\tilde{G}_{vis}$ from $p$ to $s$ with finite cost. There is an optimal path $Q = (p, q^1, \dots q^{N - 1}, s)$ through $\tilde{G}_{vis}$ with $Q[:M]$ also optimal for all $M \leq N + 1$.
\end{lemma}

\begin{proof}
Let an optimal path $Q' = (p, q^1, \dots q^{N - 1}, s)$ be fixed and arbitrary, suppose $Q'[:N]$ is suboptimal, and let $c_{path}({Q'[:N]})$ be the cost of $Q'[:N]$. Let $R$ be an optimal path from $p$ to $q^{N - 1}$. Since $Q'[:N]$ is suboptimal, we have
\begin{align}
    c_{path}(R) &< c_{path}({Q'[:N]})\\
    T + c_{path}(R) &< T + c_{path}({Q'[:N]})
\end{align}
From Lemma \ref{lemma:start_earlier_implies_arrive_at_same_time_worst_case}, we have
\begin{align}
    &\begin{aligned}
    T + &c_{path}(R) + SFT(q^{N - 1}, T + c_{path}(R), s) \leq \\
    &T + c_{path}({Q'[:N]}) + SFT(q^{N - 1}, T + c_{path}({Q'[:N]}), s)\\
    \end{aligned}\label{eqn:applied_sft_lemma}\\
    &T + c_{path}(R) + SFT(q^{N - 1}, T + c_{path}(R), s) \leq T + c_{path}(Q').\label{eqn:applied_mt_path_cost}
\end{align}
where \eqref{eqn:applied_mt_path_cost} is the result of applying \eqref{eqn:path_cost_to_mt}. Now consider a path $Q = (R, s)$, i.e. the path resulting from appending $s$ to $R$. The cost $c_{path}({Q})$ of $Q$ is
\begin{align}
    c_{path}({Q}) = c_{path}(R) + SFT(q^{N - 1}, T + c_{path}(R), s)\label{eqn:cost_of_Q}
\end{align}
Substituting \eqref{eqn:cost_of_Q} into \eqref{eqn:applied_mt_path_cost}, $c_{path}({Q}) \leq c_{path}(Q')$. Since $Q'$ is optimal, $c_{path}({Q}) \geq c_{path}(Q')$, meaning $c_{path}({Q}) = c_{path}(Q')$. Thus $Q$ is optimal. Its prefix $Q[:N] = R$ is optimal as well. Furthermore, since the edge costs in $R$ only depend on their source and destination nodes, without any dependence on an external time variable, $R$ satisfies the Markov assumption and all of its prefixes are optimal. \qed
\end{proof}

Next, we present a lemma similar to Lemma 1 from \cite{hart1968AFormalBasis}, but with an extra assumption in the lemma and additions to the proof, accounting for the fact that $\tilde{G}_{vis}$ does not satisfy the Markov assumption.

\begin{lemma}\label{lemma:open_contains_opt_node}
Let $Q = (q^0, q^1, \dots, q^N)$ be any optimal path in $\tilde{G}_{vis}$ from $p$ to $s$, with $q^0 = p$ and $q^N = s$, satisfying the following properties:
\begin{enumerate}[label=$\bullet$]
    \item $Q$ has finite cost
    \item $Q[:M]$ is optimal for all $M \leq N + 1$.
\end{enumerate}
At any iteration of Alg. \ref{alg:point_to_moving_target_planning_alg}, if $s$ is not CLOSED, there is a node $q^i$ on $Q$ in OPEN with $g(q^i) = g^*(q^i)$, where $g^*(p^i)$ is the optimal g-value of $p^i$.
\end{lemma}

\begin{proof}
Suppose $q^0$ is in OPEN. Then Lemma \ref{lemma:open_contains_opt_node} is true, because $q^0$ is on $Q$.

Suppose $q^0$ is not in OPEN, meaning it is in CLOSED, since the first expansion puts $q^0$ from OPEN into CLOSED. Let $\Delta$ be the set of closed nodes $q^i$ in $Q$ with $g(q^i) = g^*(q^i)$, and let $q^*$ be the element in $\Delta$ with largest $i$. $q^* \neq s$, because Lemma \ref{lemma:open_contains_opt_node} assumes $s$ is not in CLOSED, so $q^*$ has a successor on $Q$. Let $q'$ be the successor of $q^*$ on $Q$. We have two cases:
\begin{enumerate}[label=Case \arabic*., align=left]
    \item $q'$ is a position-node. In this case, the proof proceeds almost exactly as in the proof of Lemma 1 from \cite{hart1968AFormalBasis}, apart from one difference. \cite{hart1968AFormalBasis} claims that the optimality of $Q$ implies $g^*(q') = g^*(q^*) + \tilde{c}_{vis}(q^*, q')$. However, since path $Q$ breaks the Markov assumption, the implication does not hold. We assumed, however, that all prefixes of $Q$ are optimal, so the prefix $Q^*$ ending with $q^*$ is optimal. Therefore, the cost of $Q^*$ is $g^*(q^*)$. This means the cost of $Q' = (Q^*, q')$, i.e. the prefix of $Q$ ending in $q',$ is $g^*(q') = g^*(q^*) + \tilde{c}_{vis}(q^*, q')$. The rest of the proof proceeds as in \cite{hart1968AFormalBasis}.

    \item $q' = s$.  When we expanded $q^*$ (which must have happened for $q^*$ to be in CLOSED), we must have generated $s$ as a successor. Since $s$ is not in CLOSED, it cannot have been in CLOSED when expanding $q^*$, so if $s$ was not already in OPEN when we expanded $q^*$, we must have added $s$ to OPEN. The candidate g-value of $s$ would have been
    \begin{align}
        g_{cand}(s) &= g(q^*) + SFT(q^*, T + c_{path}(Q[:N]), s)\\
                    &= g^*(q^*) + SFT(q^*, T + c_{path}(Q[:N]), s)\\
                    &= c_{path}(Q[:N]) + SFT(q^*, T + c_{path}(Q[:N]), s)\\
                    &= c_{path}(Q).
    \end{align}
    Since $Q$ is optimal, and $g_{cand}(s)$ is the cost of $Q$, $g_{cand}(s) = g^*(s)$. Therefore at the iteration when we generated $s$ as a successor of $q^*$, we set $g(s) = g_{cand}(s)$, so $g(s) = g^*(s)$. $s$ cannot have been taken off of OPEN since $s$ is not in CLOSED, so $s = q'$ is still in OPEN with $g(q') = g^*(q')$. \qed
\end{enumerate}

\end{proof}

One might ask whether there are any optimal paths $Q$ through $\tilde{G}_{vis}$ with a suboptimal prefix $Q[:M]$, with $M < N + 1$. In the following example, we show that there are, considering $M = N$.

\begin{example}
Let $Q = (p, q^1, \dots, q^{N - 1}, q^N)$ be an optimal path in $\tilde{G}_{vis}$ from $p$ to $s$, with $q^N = s$. Consider a case where the agent can arrive at position $\tau_{\text{targ}(s)}(T^*)$ from $(p, T)$ earlier than $T^*$ while satisfying its speed limit and obstacle avoidancce constraints. In this case, the only way that $T^*$ could be the minimum interception time for $s$ is if $t^0(s) = T^*$, so if the agent arrives early, it must wait until the time window corresponding to $s$ begins. $Q$ may therefore take a suboptimal path to $q^{N - 1}$, arriving later than optimal at $q^{N - 1}$, as long as $Q$ arrives at position $\tau_{\text{targ}(s)}(T^*)$ at or before time $T^*$. In short, the optimality of $Q$ does not imply the optimality of any of its prefixes. 
\end{example}

Finally we restate Theorem \ref{theorem:optimally_solve_subproblem} from the main paper, in the context of \hyperlink{pt_mt}{\textnormal{(PT-MT)}}.

\begin{theorem}\label{theorem:feasible_and_terminate_implies_optimal}
If \hyperlink{pt_mt}{\textnormal{(PT-MT)}} is feasible, Alg. \ref{alg:point_to_moving_target_planning_alg} finds an optimal solution.
\end{theorem}

\begin{proof}
Since \hyperlink{pt_mt}{\textnormal{(PT-MT)}} is feasible, its optimal cost $T^* - T$ is finite. Applying Lemma \ref{lemma:feas_soln_implies_finite_cost_path_in_G_tw}, a path $Q$ through $\tilde{G}_{vis}$ from $p$ to $s$ has the finite cost $T^* - T$. Since $Q$ has a finite cost, Lemma \ref{lemma:path_exists_with_optimal_prefix} holds. Lemma \ref{lemma:path_exists_with_optimal_prefix} is the assumption required for Lemma \ref{lemma:open_contains_opt_node}. Lemma \ref{lemma:open_contains_opt_node} and the admissibility of our heuristic imply the corollary to Lemma 1 in \cite{hart1968AFormalBasis}. Now Theorem 1 in \cite{hart1968AFormalBasis} informs us that Alg. \ref{alg:point_to_moving_target_planning_alg} terminates with $g(s) = g^*(s)$, i.e. it terminates with an optimal path to $s$. Applying Lemma \ref{lemma:feas_soln_implies_finite_cost_path_in_G_tw} again, an optimal path in $\tilde{G}_{vis}$ to $s$ has cost equal to the optimal cost of \hyperlink{pt_mt}{\textnormal{(PT-MT)}}. Therefore when we convert $Q$ into a trajectory, we obtain an optimal solution to \hyperlink{pt_mt}{\textnormal{(PT-MT)}}. \qed
\end{proof}

\subsection{Proof of Theorems \ref{theorem:completeness_thm_feas} and \ref{theorem:completeness_thm_infeas}}
In this section, we formally present the proofs of Theorems \ref{theorem:completeness_thm_feas} and \ref{theorem:completeness_thm_infeas} from the main paper, appearing as Theorem \ref{theorem:completeness_thm_feas_appendix} and \ref{theorem:completeness_thm_infeas_appendix}, respectively.

\begin{theorem}\label{theorem:completeness_thm_feas_appendix}
If an MT-TSP-O instance is feasible, Alg. \ref{alg:dfs} finds a feasible solution.
\end{theorem}

\begin{proof}
First we show that Alg. \ref{alg:dfs} terminates. Then we show that Alg. \ref{alg:dfs} terminates by returning a feasible MT-TSP-O solution, as opposed to terminating on Line \ref{algline:term_infeas} due to an empty stack. Termination is guaranteed because Alg. \ref{alg:dfs}'s search tree is finite. We show that termination is not caused by an empty stack via an induction argument proving that there is always a prefix of a feasible MT-TSP-O solution on the stack when we check for stack emptiness on Line \ref{algline:check_stack_empty} at the beginning of a loop iteration.

\underline{Base Case}: On Line \ref{algline:init_stack}, we add the NULL trajectory to the stack, which is simply a trajectory $\tau_A$ satisfying $\tau_A(0) = p_d$. $(\tau_A, 0)$ is a prefix of all feasible MT-TSP-O solutions. A feasible solution exists by the assumption of the theorem, so $(\tau_A, 0)$ is a prefix of at least one feasible solution.

\underline{Induction Hypothesis}: Suppose $(S, \tau_A, T)$ is on the stack during some loop iteration with $(\tau_A, T)$ a prefix of a feasible MT-TSP-O solution $(\tau_A^*, T^*)$.

\underline{Induction Step}: If we do not pop $(S, \tau_A, T)$, the induction hypothesis is trivially still satisfied when we check Line \ref{algline:check_stack_empty} at the next loop iteration. Suppose we pop $(S, \tau_A, T)$. Let $p = \tau_A(T)$. Of the window-nodes intercepted by $\tau_A^*$ but not $\tau_A$, let $s'$ be the window-node intercepted earliest. Since $\tau_A^*$ travels feasibly from $(p, T)$ to $s'$, arriving at some time ${T^*}'$, Theorem \ref{theorem:optimally_solve_subproblem} implies that Alg. \ref{alg:point_to_moving_target_planning_alg} generates a trajectory $\bar{\tau}_A'$ from $(p, T)$ to $s'$, arriving with $T' \leq {T^*}'$. Under the assumptions that the targets move no faster than the agent's maximum speed and do not enter the interior of obstacles during their time windows, it is feasible for an agent trajectory to follow $\tau_{\text{targ}(s')}$ from $T'$ to ${T^*}'$. Therefore we can construct a feasible MT-TSP-O solution $(\tau_A^{**}, T^*)$ by having the agent follow $\tau_A$ until time $T$, follow $\bar{\tau}_A'$ until $T'$, follow $\tau_{\text{targ}(s')}$ from $T'$ to ${T^*}'$, and follow $\tau_A^*$ from ${T^*}'$ to $T^*$. If $s' = s_d$, we return a feasible solution on Line \ref{algline:return_mt_tsp_o_soln} and we never check Line \ref{algline:check_stack_empty} again. If $s' \neq s_d$, we push $(S', \tau_A', T')$ onto the stack, where $\tau_A'$ is the concatenation of $\tau_A$ with $\bar{\tau}_A'$, and $(\tau_A', T')$ is a prefix of $(\tau_A^{**}, T^*)$.

By induction, there is always a prefix of a feasible MT-TSP-O solution on the stack. As we stated at the beginning of the proof, this guarantees that Alg. \ref{alg:dfs} terminates and returns a feasible solution.
\end{proof}

\begin{theorem}\label{theorem:completeness_thm_infeas_appendix}
If MT-TSP-O is infeasible, Alg. \ref{alg:dfs} terminates and reports infeasible in finite time.
\end{theorem}

\begin{proof}
Alg. \ref{alg:dfs}'s search tree is finite, so Alg. \ref{alg:dfs} must terminate. It cannot terminate by returning a feasible solution on Line \ref{algline:return_mt_tsp_o_soln} because no feasible solution exists, so it must terminate on Line \ref{algline:term_infeas} and report infeasible. 

Recall that before executing Alg. \ref{alg:dfs}, we must compute the LFDT from every window-node to every other window-node, using Alg. \ref{alg:point_to_moving_target_planning_alg}. We must therefore ensure that Alg. \ref{alg:point_to_moving_target_planning_alg} reports infeasible in finite time for an infeasible query $(p, T, s, G_{vis}, \Lambda_{vis})$. Line \ref{algline:check_before_add_to_open} prevents Alg. \ref{alg:point_to_moving_target_planning_alg} from adding nodes $v'$ into OPEN with $g(v) + \tilde{c}_{vis}(v, v', T + g(v)) > t^f(s) - T$, where $v$ is a predecessor of $v'$. Therefore the search will never add nodes to OPEN more than $M = \frac{t^f(s) - T}{\delta}$ steps away from $p$, where $\delta$ is a lower bound on edge costs in $\tilde{G}_{vis}$. We can use the upper bound $M$ with the proof of Theorem 1, Case 2 in \cite{hart1968AFormalBasis} to prove that for infeasible queries, Alg. \ref{alg:point_to_moving_target_planning_alg} will terminate and report infeasible.
\end{proof}

\end{subappendices}

\end{document}